\documentclass{ecai} 

\usepackage{latexsym}
\usepackage{amssymb}
\usepackage{amsmath}
\usepackage{amsthm}
\usepackage{booktabs}
\usepackage{enumitem}
\usepackage{graphicx}
\usepackage{color} 
 
\usepackage{subcaption} 

\usepackage{algorithm}
\usepackage{algpseudocode}
\usepackage{multirow}
\usepackage{scalerel}

\begin{document}

\begin{frontmatter}
%
%
  \title{Graph Classification with GNNs : \\ Optimisation, Representation \& Inductive Bias}
  
\author[A]{\fnms{P. Krishna}~\snm{Kumar}}

\author[A,B]{\fnms{Harish G.}~\snm{Ramaswamy}}  

\address[A]{Department of Computer Science \& Engineering, IIT Madras}
\address[B]{Department of Data Science \& Artificial Intelligence, IIT Madras }
\address{\texttt{\{pkrishna,~hariguru\}@cse.iitm.ac.in}}

\begin{abstract}

Theoretical studies on the representation power of GNNs have been centered around understanding the equivalence of GNNs, using WL-Tests for detecting graph isomorphism. 
In this paper, we argue that such equivalence ignores the accompanying optimization issues and does not provide a holistic view of the GNN learning process. 
We illustrate these gaps between  representation and  optimization with examples and experiments.
We also explore the existence of an implicit inductive bias (e.g. fully connected networks prefer to learn low frequency functions in their input space) in GNNs, in the context of graph classification tasks. 
We further prove theoretically that the message-passing layers in the graph, have a tendency to search for either discriminative subgraphs, or a collection of discriminative  nodes dispersed across the graph, depending on the different global pooling layers used. 
We empirically verify this bias through experiments over real-world and synthetic datasets.
Finally, we show how our work can help in incorporating domain knowledge via attention based architectures, and can evince their capability to  discriminate coherent subgraphs.

\end{abstract}
\end{frontmatter}

\newcommand\todo[1]{\textcolor{red}{[[#1]]}}
\newcommand\smz[1]{$\pm$ \scriptsize{{#1}}} 

\newcommand\supplementaryOn[1]{#1}

\newcommand\captionz[1]{\vspace{2mm}\caption{#1} \vspace{2mm}}

\newtheorem{theorem}{Theorem}
\newtheorem{lemma}[theorem]{Lemma}
\newtheorem{corollary}[theorem]{Corollary}
\newtheorem{proposition}[theorem]{Proposition}
\newtheorem{fact}[theorem]{Fact}
\newtheorem{assumption}[theorem]{Assumption}
\newtheorem{definition}{Definition}
\newtheorem*{remark}{Remark}

\newenvironment{hproof}{%
  \renewcommand{\proofname}{Proof sketch}\proof}{\endproof}

\newcommand\numberthis{\addtocounter{equation}{1}\tag{\theequation}}

\newcommand\real{\mathbb{R}}
\renewcommand\natural{\mathbb{N}}

\newcommand\relu{\mathrm{ReLU}}
\newcommand\T{\mathsf{T}}
\newcommand\mlp{\mathrm{MLP}}

\newcommand{\cg}{\textnormal{\textsl{g}}}
\newcommand\multiset[1]{\{\!\!\{ #1 \}\!\!\}}
\newcommand\ngbr[1]{\mathcal{N}(#1)}

\newcommand\dd[2]{\frac{\partial {#1}}{\partial {#2}}}
\newcommand\DD[2]{\frac{\mathrm{d} {#1}}{\mathrm{d} {#2}}}
\newcommand\loss{\mathrm{\mathbf{L}}}

\newcommand\attn{\mathbf{a}}

\newcommand\wgt{\mathbf{w}}

\newcommand\aggr{\hat{v}}

\newcommand\tp{^\intercal}

\newcommand\data{\mathcal{D}}
\newcommand\dataO{\mathcal{D}_0}
\newcommand\dataI{\mathcal{D}_1}
\newcommand\dataP{\mathcal{D}_\perp}

\newcommand\struct{\mathcal{S}}
\newcommand\vjv[1]{v\tp {#1}}
\newcommand\vstarv[1]{{v^*}\tp {#1}}
\section{Introduction}

Graph Neural Network's (GNNs) have enabled end-to-end learning over \textit{relational} data due to differentiable loss functions, that can be trained with non-linear components like multi-layer perceptrons.
Several real world applications such as fake news detection \citep{DBLP:conf/acl/MehtaPG22}, physical simulation \citep{DBLP:conf/nips/BattagliaPLRK16}, traffic delay estimation \citep{DBLP:conf/cikm/Derrow-PinionSW21}, and fraudulent transactions prediction\citep{DBLP:journals/tkde/ChengWZZ22} have GNNs as a crucial component. 

Graph classification is one of the most common downstream graph neural processing applications\citep{DBLP:conf/iclr/ErricaPBM20}.
While different GNN operators update node level features via message-passing, 
the graph level predictions are done by pooling the member nodes into a single unified representation. 
This pooling is either done by coarsening functions that gradually reduce the size of the graph \citep{DBLP:conf/icml/BianchiGA20}, or with the help of global pooling methods like average, max, sum \citep{DBLP:conf/nips/DuvenaudMABHAA15}.

The success of GNNs has also led to several attempts toward defining theoretical boundaries of what GNNs can and can not do. 
The strengths and weaknesses of graph neural networks have been extensively evaluated in terms of their \textit{representation} capabilities. 
Most studies have focused on the capability of message-passing networks using Weisfeiler-Lehman test, which is constrained by its limitations in distinguishing isomorphic graphs  \citep{DBLP:books/cu/G2017}. 
The proposed architecture of Graph Isomorphism Network \citep{DBLP:conf/iclr/XuHLJ19} implements an MLP to model injective function, and be as powerful as 1-WL test. 
Higher order generalization of WL tests, using a combination of equivariant and invariant functions, is shown to surpass limitations of simple GNNs that use message-passing \citep{DBLP:conf/aaai/0001RFHLRG19}. 
A disparity exists between the theoretical understandings and optimization practices within GNNs. 
For instance, previous works indicate that subgraph-based counting/classification is inherently unattainable for subgraphs exceeding a size of three \citep{DBLP:conf/nips/Chen0VB20}. However, there are empirical findings to challenge these constraints, primarily due to the relaxation of restrictive assumptions, such as fixed-node features derived from a countable set.

In this work, we introduce an implicit bias inherent in graph convolutional networks when integrated with self-attention-based global pooling functions. To our knowledge, this marks the initial investigation characterizing the graph-classification task with a focus on attention pooling layers. We empirically explore the full spectrum of attention, ranging from average pooling (equivalent to no attention) to max pooling (singular focus). Additionally, we substantiate our empirical findings with theoretical insights and conduct an in-depth analysis of attention-based global pooling.
We summarize our main contributions as follows:
\begin{itemize}
\item \textbf{Develop an experimental setup to elicit GNN's implicit bias: } 
We introduce a dataset with unique characteristics of \textit{occurrence} and \textit{connectedness}, that underscores the ambiguity inherent in GNN models for graph classification tasks. Our approach provides empirical insights into deducing the implicit bias of GNNs by closely examining their nuanced behavior on synthetic grid-graph world.
\item \textbf{Theoretically analyze of attention based global pooling:} We employ gradient-flow \citep{DBLP:conf/nips/ChizatB18} to prove the existence and uniqueness of final vector, towards which our learned model parameters align. Our main result formalizes the biased nature of attention pooling, and shows it is preferring to use \textit{closely-connected-substructures} as discriminative features, rather than collection of nodes that may be dispersed across the graph. 
\item \textbf{Empirical validation:} We study realworld datasets with planted ground-truth, to show the ambiguity caused by the inherent bias. 
We isolate the effect of the linear classifier from the message-passing layers. Lastly, our study underscores practical implications in tasks where domain knowledge informs graph classification. It aids in discerning whether classification hinges on the presence of coherent subgraphs or on merely the presence of nodes, irrespective of their neighbors.
\end{itemize}

\section{Subgraph Based Graph Classification Task}
\label{sec:subgraph-find}

Graph Neural Networks learn to infer graph labeling functions by leveraging latent representation of nodes, that extend to the entire graph structure. 
A trained classifier $h(X,A)$ operates on both node features and graph topology, producing a class label as output. However, the mechanisms driving classification are not fully understood. 
It remains unclear whether classifications arise from the presence or absence of specific structural components within the graph, such as nodes belonging to distinct classes forming unique subgraphs, or if predictions stem from numerical decision boundaries imposed over aggregated graph representations. 
Recent efforts in \textit{explaining} GNN predictions aim to identify subgraphs or disjointed nodes that cause maximal change to the specific predictions  \citep{DBLP:conf/iclr/SchlichtkrullCT21,DBLP:conf/nips/YingBYZL19,DBLP:conf/log/AmaraYZHZSBS022}. 
Despite these endeavors, whether GNN architectures inherently learn to discern specific graph traits for prediction purposes, remains an open question.
We aim to address such knowledge gaps through the subsequent sections of this paper.

We establish a common notation and terminology used throughout this paper.
$\data$ denotes a dataset of attributed graphs with $N$ nodes in each graph. 
Each graph  $G \in \data$ is represented by a tuple of its node-feature matrix and edge adjacency matrix $(X,A)\in (\real^{N \times d}, \{0,1\}^{N \times N})$. 
For natural number $K \in \mathbb{N}$,  $[K]$ denotes sequence $\{1,2,3 \ldots K\}$. $\struct$ denote an ordered tuple $= (s_1, s_2, \ldots s_M) \in [K]^M $.
Let   $G^* = (X^*,A^*)\in (\real^{M \times d}, \{0,1\}^{M \times M})$ be a special graph with $M$ nodes, where $M<N$. In our graph classification problem, the set of graphs in $\data$ are labelled based on their relation to $G^*$, and our goal is to learn a model that can learn to map a graph to its correct label. Note that $G^*$ is part of the data generation process, is not known to the learning algorithm, and is not necessarily required to be recovered by the learned graph classification model.
We denote the node features and corresponding adjacency matrix of subgraph $G^*$ by $X_{[\struct,:]} \in \real^{M\times d}$ and $A_{[\struct,\struct]} \in \{0,1\}^{M\times M}$. 
The matrix subscripts corresponds to rows and columns of indices in the ordered tuple $\struct$.

To illustrate this concept via a simple example, let's examine a dataset $\data$, that comprises of numerous graphs $G \in \data$. 
As per the law of dichotomy, each graph can either contain or lack a subgraph representing an orange-blue-green (O-B-G) chain 
\big(\scalerel*{\includegraphics{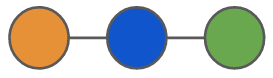}}{B}\big), where colors denote the respective node features. 
Let us define a labeling function $y(G)$, such that $y(G) = +1$ if the O-B-G subgraph is present in $G$, and $y(G) = -1$ otherwise. Consequently, the dataset $\data$ can be partitioned into two disjoint subsets, $\data_{1}$ and $\data_{-1}$, based on their respective class labels.

\subsection*{Can GNNs detect presence and absence of subgraphs?}\label{sec:detect}
To ascertain the capability of GNNs in discerning graphs based upon the existence or non-existence of a specified subgraph $G^*$, we have devised a simple yet informative experiment. 
We generated grid-graphs measuring $12\times 12$ and randomly assigned one-hot embedding to its node features, corresponding to each node taking one of 4 colors. 
Subsequently, we delineated two distinct graph sets, denoted as $\data_1$ and $\data_{-1}$, contingent upon the presence of a subgraph comprising a special chain of length equal to three.
The precise method used to generate $\data_1$ and $\data_{-1}$, and to ensure the presence or absence of this subgraph is detailed in Section \ref{sec:setting}.
All graphs within $\data_1$ were assigned a ground-truth label of $y=1$, while those in $\data_{-1}$ were labeled with $y=0$.

\begin{table}[t!]
  
    \renewcommand*{\arraystretch}{1.3}
    \centering
\captionz{
Performance of GNN models that are trained on $12\times 12$ with $\data_1, \data_{-1}$, and tested over unseen grid-graphs of varying sizes.
Results show that GNNs are capable of learning to distinguish graphs based on presence/absence. (Sec.\ref{sec:detect}) \label{tab:rep_empirical} } 
\begin{tabular}{|c|c|c|c|c|c|}
\hline
Model &
  \begin{tabular}[c]{@{}c@{}}Pooling \\ [-0.3em] func\end{tabular} &
  \begin{tabular}[c]{@{}c@{}}Training \\ [-0.3em] Loss\end{tabular} &
  \begin{tabular}[c]{@{}c@{}}Train \\ [-0.3em] Accuracy\end{tabular} &
  \begin{tabular}[c]{@{}c@{}}Test Acc\\ [-0.3em] 12 x 12\end{tabular} &
  \begin{tabular}[c]{@{}c@{}}Test Acc\\ [-0.3em] 13 x 13\end{tabular} \\ \hline
\multirow{3}{*}{GCN} & MAX  & 0.0766 & 0.987 & 0.979 & 0.991 \\ \cline{2-6} 
                     & AVG  & 0.0451 & 1.000 & 0.998 & 0.517 \\ \cline{2-6} 
                     & ATTN & 0.0948 & 0.978 & 0.923 & 0.986 \\ \hline
\multirow{3}{*}{GAT} & MAX  & 0.0715 & 0.991 & 0.968 & 0.994 \\ \cline{2-6} 
                     & AVG  & 0.1213 & 1.000 & 0.998 & 0.675 \\ \cline{2-6} 
                     & ATTN & 0.0293 & 0.993 & 0.973 & 0.995 \\ \hline
\end{tabular}
\end{table}

After training GNN models on the datasets $\data_1$ and $\data_{-1}$, we tested an independent set of grid-graphs, generated with dimensions $13\times 13$ and $12\times 12$. 
Table \ref{tab:rep_empirical} presents the empirical results, demonstrating the efficacy of popular GNN models, such as GCN and GAT, when coupled with various global pooling mechanisms. 
Notably, these models showcase remarkable generalization capabilities over previously unseen test data. Across all model-pooling combinations, the test accuracy consistently surpass the $90\%$ threshold for both grid-graph sizes, except for the average pooling mechanism on the $13\times 13$ grid. Given that all models achieve high accuracy and are optimized until attaining near-zero training loss (binary cross-entropy), we can say that GNNs possess the empirical capacity to differentiate between the presence and absence of the designated subgraph.

In the following sections, we will place this empirical observation in the context of theoretical results and attempt to reconcile multiple views.

\section{Limits of Representation Power Arguments}
\label{sec:limits}

There is an apparent difference in the theoretical representation pertaining to power of neural networks and their practical implication in real settings. 
For example, deeper neural networks have a definite advantage over shallow networks. It is known that networks with polynomial increase in depth can approximate functions that exponentially growing width networks cannot \cite{DBLP:conf/colt/Telgarsky16}. 
Simple architectures like deep feed-forward networks with piece-wise linear activation functions (like ReLU) can distinguish exponentially greater input regions due to the increased depth \cite{DBLP:conf/nips/MontufarPCB14}.  
However, the mere ability of deeper networks to represent complex functions does not guarantee that they will be learned by gradient descent. For example, it has been observed that the number of distinct linear regions in ReLU based networks grows linearly along any single dimension. 
It highlights the gap between theoretical possibilities and empirical observations while training neural networks.

Similar to feed-forward networks, graph neural networks are shown to have restricted representation power, and their theoretical limitations have been an area of active research \citep{DBLP:journals/corr/abs-2204-07697}. 
Message passing neural networks (MPNNs), encompassing different convolution operators for GNNs, are theoretically as powerful as 1-Weisfeiler Lehmann test for distinguishing non-isomorphic graphs \citep{DBLP:conf/iclr/XuHLJ19}. 
In similar vein, higher order generalizations of $k$-WL test (for $k>2$) are known to be strictly more powerful than 1-WL and capable of learning sub-graph patterns of maximally $k$-nodes \citep{DBLP:conf/aaai/0001RFHLRG19}.
In the context of subgraph-counting, MPNNs can only count a given subgraph (or induced-subgraph) if there are 3 or fewer nodes in the subgraph (barring exceptions of star-shaped subgraphs) \citep{DBLP:conf/nips/Chen0VB20}. 
As earlier, the higher order GNNs based on $k$-WL are shown to be powerful enough to count subgraphs of utmost $k$ nodes.

The negative results from GNN's representation-theory are not as discouraging as they seem on the surface. For example, the negative result claiming that GNNs \textit{cannot} detect presence/absence of a subgraph larger than size 4 \citep{DBLP:conf/nips/Chen0VB20}, really corresponds to saying that there \textit{exists} a distribution over graphs and labels for which the best GNN based classifier has low accuracy. We argue that such adversarial distributions are not indicative of real data. For example, as shown in Table \ref{tab:rep_empirical}, GNNs are able to learn to classify graphs based on presence/absence of subgraphs.

Despite these theoretic limitations, GNNs have proven to be empirically very successful in most graph based tasks, and achieved state-of-the-art performance in many real-world applications \citep{DBLP:conf/acl/MehtaPG22,DBLP:conf/nips/BattagliaPLRK16,DBLP:conf/cikm/Derrow-PinionSW21,DBLP:journals/tkde/ChengWZZ22}. In this paper we argue that studying GNN based graph classification algorithms in the context of optimisation and generalisation is still an open and interesting objective. 

\section{Inductive Bias of GNNs Under Gradient Descent}
\label{sec:inductivebias}

\begin{table}
   \captionz{Characteristics of the partitions of dataset $\data$ with respect to nodes/edges of a subgraph $G^*$}
   \renewcommand*{\arraystretch}{1.3}
    \label{tab:d0d1}
    \centering
    \begin{tabular}{|c|c|c|}
    \hline
      Partition & \textit{Occurrence} & \textit{Connected-ness}  \\
  \hline
    $\data_1$     & $\checkmark$ &  $\checkmark$ \\
    $\data_0$     & $\times$    & $\times$  \\
    $\dataP$  & $\checkmark$  &  $\times$  \\
  \hline
    \end{tabular}
\end{table}

\begin{figure}[t!]
    \centering
    \begin{subfigure}[t]{0.115\textwidth}
        \centering
        \includegraphics[width=0.9\textwidth]{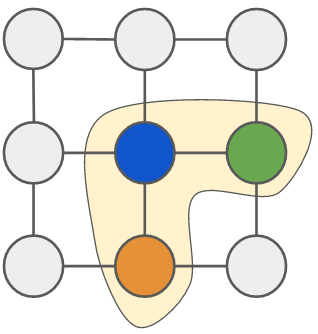}
        \captionz{$G \in \data_1$} \label{fig:toy.d1}
    \end{subfigure}
        \begin{subfigure}[t]{0.24\textwidth}
        \centering
        \includegraphics[width=0.845\textwidth]{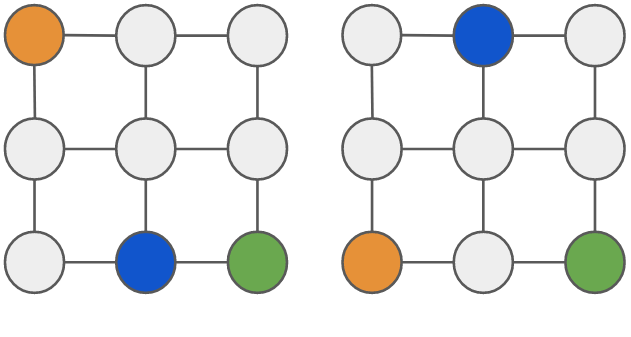}
        \captionz{$G_1, G_2 \in \dataP$} \label{fig:toy.dperp}
    \end{subfigure}
    \begin{subfigure}[t]{0.118\textwidth}
        \centering
        \includegraphics[width=0.9\textwidth]{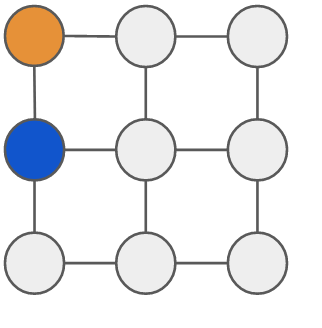}
        \captionz{$G \in \data_0$} \label{fig:toy.d0}
    \end{subfigure} 
    \captionz{Sample graphs from each partition of data $\data$. 
          (a) Subgraph $G^*$ 
          \includegraphics[height=0.7\baselineskip]{fig_subg.png}
        is present in all $G \in \data_1$.
          (b) shows the presence of all nodes of $G^*$ in $G \in \dataP$ but none of the nodes are mutually adjacent to each other.
          (c) among graphs $G \in \data_0$, subgraph $G^*$ is only partially present (of size 1 or 2).}\label{fig:toy}  \vspace{4mm}
\end{figure}

In this section, we explore if there exists a proclivity for the GNN to learn a particular category of functions.
We study one such characteristic tendency of GNNs, where it exhibits distinct inclinations towards learning specific function types, when convolution operator is coupled with varied pooling mechanisms.
Consider the graph classification task on dataset created around subgraph $G^*$ as seen in the previous sections. 
For the purpose of illustration, let $G^*$ be the O-B-G \textit{orange-blue-green} chain 
and embedded in graphs of a fixed-size grid-graph dataset
 $\data$. \footnote{Arguments can trivially extend to graphs of varying sizes and multiple classes classification. There is nothing sacrosanct about the given setting.}
Let $\data_1$ be the partition of $\data$ where subgraph $G^*$ is \textit{present}, for example Fig.\ref{fig:toy.d1} contains \textit{O-B-G}.
Similarly, let $\data_0$ be the set of all graphs within $\data$ where $G^*$ is only \textit{partially} present, as in at least one node of $G^*$ is missing. 
For instance, Fig.\ref{fig:toy.d0} does not contain the green node.
All the remaining graphs together form $\dataP$, which implies the graphs in $\dataP$ consist of all the nodes of $G^*$, (O,B,G), but do not contain all edge of $G^*$, for example Fig.\ref{fig:toy.dperp} shows that no two colored nodes are adjacent.

In the context of this example, let \textit{occurrence} denote the presence of `all' attributed-nodes of $G^*$ and \textit{connected-ness} denote the presence of `all' edges of $G^*$.
Table \ref{tab:d0d1} characterizes the mutually exclusive partitions of $\data_1, \data_0$ and $\dataP$ along these two properties and  highlights their commonalities and differences.
While $\dataP$ shares its \textit{occurrence and connected-ness} property with both the other partitions, it is not obvious whether it can be associated as more close to either $\data_1$ or $\data_0$. 
To verify this empirically, we experiment with a GNN trained over $\{\data_1 \cup \data_0\}$ and further analyze its behavior upon $\dataP$.
Say we assign label $y=1$ to $\data_1$ and $y=0$ to $\data_0$, and train a GNN over set $\{\data_1 \cup \data_0\}$ to achieve perfect accuracy. 
It remains non-obvious to predict whether graphs of $\dataP$ will be classified as label $y=1$ or $y=0$.

\begin{figure}[t]
    \centering
       
         \begin{subfigure}[t]{0.30\textwidth}
        \centering
        \includegraphics[width=0.99\textwidth]{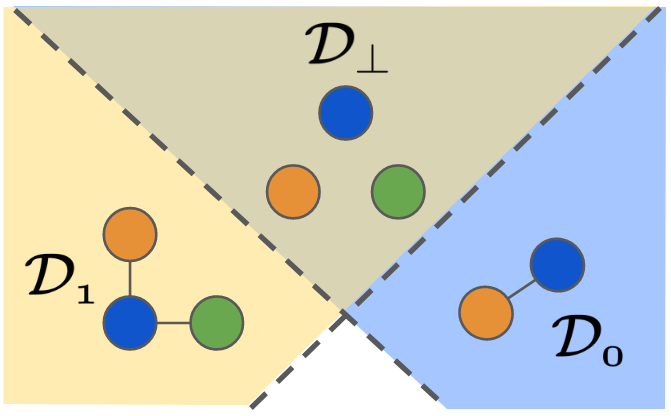}
    \end{subfigure}
    \begin{subfigure}[t]{0.108\textwidth}
        \centering
        \includegraphics[width=0.9\textwidth]{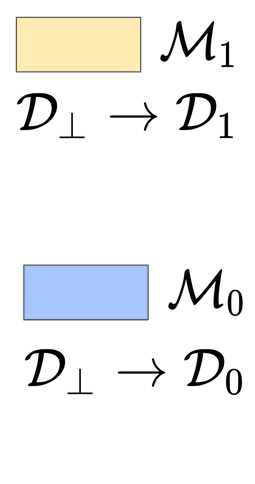}
    \end{subfigure}
    \captionz{
    Multiple legal decision-boundaries (dotted lines) that can be learned while training GNNs on $\{\data_1 \cup \data_0\}$. 
    $\dataP$ can be classified along with either $\data_1$ or $\data_0$. 
    We show that \textit{(i)} the graph convolution operator with global average pooling classifies $\dataP$ as $\data_1$ and behaves as $\mathcal{M}_1$ family of functions denoted by  
the orange 
\includegraphics[height=0.7\baselineskip]{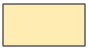}
decision regions,  
while \textit{(ii)} GCN with attention based pooling classifies $\dataP$ as $\data_0$ and learns $\mathcal{M}_0$ denoted by blue 
\includegraphics[height=0.7\baselineskip]{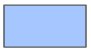} 
decision region. \label{fig:intro}}
\vspace{4mm}
\end{figure}

\subsection{Ambiguity in Graph Classification: Figure \ref{fig:intro}}
A GNN that has not seen graphs from the set $\dataP$ while training on $\{\data_1 \cup \data_0\}$ could very well assign it either label $1$ or $0$. 
However, both these assignments indicate different inductive biases.

\noindent
\textbf{Case 1.} $(\dataP \to \data_1)$  
Suppose, a trained GNN consistently classifies $\dataP$ to be same as $\data_1$, instead of $\data_0$, it shows that the model looks for the \textit{occurrence} of all nodes in $G^*$. 

\noindent
\textbf{Case 2.} $(\dataP \to \data_0)$
Conversely, if a trained GNN aligns $\dataP$ with $\data_0$, it indicates that the model emphasizes on \textit{connected-ness} of all nodes  in $G^*$.

\noindent
\textbf{How to Deduce Implicit Bias?} 
Let's understand this bias using the example depicted in Figure \ref{fig:toy}, let us assume $\data$ to be the universe of all fixed-size grid-graphs and $\data_1, \data_0$  to exhaustively contain all graphs satisfying the above discussed properties as in Table \ref{tab:d0d1}.
Now, if a trained GNN classifies graph in $dataP$ as $\data_1$ rather than as $\data_0$, it implies that the GNN is looking for presence of all the three (\textit{orange,blue,green}) nodes in each graph. 
Otherwise, there will exist $G_1, G_2 \in\data_0$ such that their nodes together will contain all $\{O,B,G\}$ but separated among the two graphs $G_1$ and $G_2$.
On the other hand, if a GNN classifies $\dataP$ as $\data_0$ instead of $\data_1$, it 
implies the GNN is looking for connection between the three (\textit{O-B-G}) nodes.
In other words, since $\data_0$ consists of all possible disintegrated subparts of $G^*$, the only reason $\dataP$ can be classified as $\data_0$ is due to the lack of $G^*$ being present as such.

\subsection{Grid-Graphs Dataset}  
\label{sec:OvC}
To empirically validate the capability of GNNs to distinguish   \textit{occurrence} versus \textit{connectedness}, we start with a simple and controllable setting of grid-graphs with node-features, as discussed above. 
Each node in the graph is assigned a random one-hot feature vector, referred to as `node-color', from a finite set of vectors.
We further derive two partition sets, $\data_1$ and $\data_0$, by selectively planting certain subgraph $G^*$ within the grid-graph. 
We assume our universe of data $\data$ restricted to the grid-graphs $\mathcal{G}$ of fixed size, and assign each node a random feature from the standard basis-vector set $\mathcal{C}=\{e_i\}$ for $i \in [K]$ classes. 
We further create a three-node graph $\mathcal{G}^*$, and assign each of its nodes distinct features, 
denoted by $\mathcal{C}'=\{e_j\}$. 
The partition $\data_1$ is generated by planting the subgraph $\mathcal{G}^*$ in grid-graphs $G \in \data$. We randomly choose an anchor node $v_1 \in \mathcal{G}$ along with its two other neighbors $v_2,v_3 \in \ngbr{v_1}$, and then alter their features such that $\mathcal{G}^*$ is contained in $\mathcal{G}$.

Similarly, graphs in partition $\data_0$ are generated by \textit{partially} planting $G^*$, that is by taking a strict subset $\tilde{\mathcal{C}} \subset \mathcal{C}'$ of size $1$ or $2$ and replacing node features of random nodes in $\mathcal{G}$  with $\tilde{\mathcal{C}}$. 
The graphs in $\dataP$ are generated by randomly choosing $|\tilde{\mathcal{C}}|$ number of mutually non-adjacent nodes, and sequentially assigning them node features from $\tilde{\mathcal{C}}$. 
Note that the chosen universe of $\data$ is small enough for complete enumeration and for exhaustive generation of $\data_1$, 
i.e., any graph that satisfies the said characteristic properties will be present in $\data_1$.
The datasets size is restricted by number of nodes in the grid-graphs that we sequentially choose as anchor nodes to plant $G^*$.
To remove any class label-imbalance, we maintain parity between the sizes of $\data_1$ and $\data_0$.
\\
\textbf{\textit{Labeling function}}: 
The graphs in the two partitions are assigned binary labels for the supervised training of the GNNs. Without loss of generality, $\data_1$ are assigned label $y=1$, and $\data_0$ graphs with partially  planted $\tilde{\mathcal{C}}$ are assigned label $y=0$.
Since $\dataP$ is never used for training GNN models, it is not assigned any label.
In case of 3-node $G^*$, after 1-step of propagation, the central node's representation will have a positive component along v*; the proof of its existence and uniqueness follows.

\subsection{Model Baselines}\label{sec:baselines}
We study the behaviors of two prominent models, the graph convolution network (GCN) \citep{DBLP:conf/iclr/KipfW17} and the graph attention network (GAT) \citep{DBLP:conf/iclr/VelickovicCCRLB18,DBLP:conf/iclr/Brody0Y22}.

Along with these message-passing layers, we use common global pooling functions, encompassing summation, averaging, and attention-based aggregation.
For any graph $G=(X,A)$, GCN updates are given by $f_{\mathrm{GCN}}(X,A;W) = \hat{A}XW$,
where $\hat{A}$ is the degree normalized adjacency matrix. 
After one step of propagation, the updated representation of node $i$ can be given by $v_i = W\sum_{j\in \ngbr{i}} x_j$. 
Further, all the nodes of the graph $\{v_1, v_2 \hdots v_N \}$ are pooled using average, maximum or attention mechanism. 
For $(A,X) \in \real^{N\times N} \times \real^{N\times k}$, the end-to-end equation for graph classification using GCN and average global-pooling can be written as
\begin{equation} 
f_{\mathrm{GCN+AVG.}}(X,A;~\Theta,\wgt) = \mathbf{1}\tp \cdot \sigma(\hat{A}X\Theta) \cdot \wgt   \label{eqn:GCNupdate}
\end{equation}
where $\Theta \in \real^{k\times d}$ is the linear-transform-parameter of GCN, $\mathbf{1}$ is column of all $1$'s, and $\wgt \in \real^d$ is the linear classifier. 
Similarly, the end-to-end model equation for GCN with global attention pooling can be written as
\begin{equation}
f_{\mathrm{A}}(X,A;~\Theta,\attn,\wgt)= \mathrm{sftmx}\left[\big<\sigma(\hat{A}X\Theta),\attn\big>\right] \sigma(\hat{A}X\Theta) \cdot \wgt    \label{eqn:GCN_ATTNupdate}
\end{equation}
where $\attn \in \real^d$ parameter is used to calculate the attention scores by taking \textit{softmax} over the projection of nodes. 
Attention based global pooling is a special case of the self-attention pooling mechanism \citep{DBLP:conf/icml/LeeLK19}, where the representation of the entire graph is collapsed into one single vector.

\subsection{Empirical Evidence for Inductive Bias}
We empirically show that there exists a clear bias in GNNs, while classifying graphs from the above data partitions.
We start with the most popular GNN architectures, graph convolution networks (GCN) and graph attention networks (GAT), along with most common global pooling methods like average, maximum and attention based readout function. 
We use the dataset described in previous Section \ref{sec:OvC}, $12\times 12$ grid graphs, partitioned into $\data_1, \data_0, \dataP$. 
All models, based on GCN/GAT layers coupled with MAX/AVG/ATTN global pooling layers, are trained to near perfect training accuracy (greater than $99\%$). 
While training, only $\data_1$ and $\data_0$ are exposed to the models and $\dataP$ remains inaccessible. 
We then test each of these \textit{trained} models on test data $\dataP$ to analyze what labels $\hat{y}\in \{0,1\}$ are assigned to graphs $G\in \dataP$. \\
\textbf{Inductive bias from observations in Table \ref{tab:rep_empirical}:} 
We observe that both GCN and GAT, with average pooling classify almost all unseen graphs as label $\hat{y}=1$, while attention based pooling classify almost all unseen graphs as $\hat{y}=0$. 
Assigning such skewed proportions of nodes with selective label $\hat{y}=1$ or $\hat{y}=0$, is a clear indicator that both models have preference to learn different functions. 
While this distinctive behavior shows an interesting biased behavior, we will see in later Section \ref{sec:setting} that this behavior translates to model's preference of subgraph selection.

\section{Analyzing Attention Based Global Pooling}

We analyze the behavior of GCN based models with attention pooling, and characterize it for grid-graph dataset.
We start with formalizing notations for the three data partitions $\data_1, \data_0$ and $\dataP$ below, along with precise definitions for \textit{occurrence} and \textit{connected-ness}.

\begin{definition}[Data Partitions]
A graph dataset $\data$ can be partitioned into three mutually exclusive sets on the basis of subgraph $G^* = (X^*, A^*)$ as follows:
\begin{itemize}
  \item $\data_1$ is the set of graphs $\{G \in \data ~\vert~ \exists \struct = (s_1, s_2 \ldots s_M) \in [N]^M\}$, such that $X_{[\struct,:]} = X^*$ and $A_{[\struct,\struct]} \geq A^*$ for some ordered tuple $\struct$ (here matrix comparison is element-wise)
  \item $\data_0$  is the set of graphs $\{G \in \data ~\vert~ X_{\struct,:} \neq X^*, ~\forall \struct \in [N]^M\}$.
  \item $\data_\perp$  is the set of graphs $\{G \in \data ~\vert~  ~\exists \struct\in [N]^M ~\text{s.t}~ X_{\struct,:} = X^* ~\text{and}~ G \notin D_1\}$
\end{itemize}
\end{definition}

\noindent
In other words, $\data_1$ contains $G^*$ as its subgraph, $\data_0$ subgraphs do not contain at least one node from $G^*$, and lastly $\dataP$ subgraphs have all the nodes of $G^*$, but does not contain $G^*$ subgraph.

\begin{definition}[Occurrence and Connection]
A subgraph $G^*=(X^*, A^*)$ is said to \textit{occur} in $G=(X,A)$ iff there exists an ordered tuple $\struct = (s_1, s_2 \ldots s_M) \in [N]^M$, such that $X_{[\struct,:]} = X^*$.
Also, $G$  with $G^*$ is said to be connected iff $A_{[\struct,\struct]} \geq A^*$ for some tuple $\struct$.
\end{definition}

\noindent
\textbf{Graph Convolutional Networks (GCN):}
Our theoretical analysis is restricted to a simpler setting of GCN after one step of message-propagation. GCN model updates the node-representations as 
$ x_i = \frac{1}{\sqrt{d_i d_j}} W \sum_{j \in \ngbr{i} \cup i} x_j$,
where $W$ is its learnable parameter and $d_id_j$ is a degree based normalizing coefficient. 
The summation of neighboring nodes after one-hop message-passing can be denoted by matrix product $V = AX$. Ignoring the degree normalization,  here each row  $v_i = V_{i,:}$ is the updated feature vector of node $i$.

\noindent
\textbf{GCN with Global Attention Pooling (ATTN): }
While the graph convolution operator does the job of aggregating immediate neighbors, the global pooling functions, also known as READOUT functions, aggregate all node's representation to get a single vector for the entire graph. 
We analyze the behavior of attention based global pooling, which  is calculated by the weighted summation over nodes as  
$h_G = \sum_{i=1}^N \alpha_i w \big( \sum_{j \in \ngbr{i}} x_j \big) =  \sum_{i=1}^N \alpha_i v_i = \hat{v}$, 
with two learnable parameters: $\attn \in \real^d$  and  $\wgt \in \real^d$.

\noindent
The following lemma shows that for each graph $G$, there exists a node $s$ (with its node  representation $v^*$) that contains aggregated information from all the nodes of subgraph $G^*$ after one-step of GCN aggregation. 

\begin{lemma}[Existence of $v^*$] \label{lem:exist}
There exists $v^* \in \real^d$, such that for any graph $G = (X,A) \in \data_1$,  there is an ordered tuple $\struct \in [N]^M$ satisfying \mbox{$\big<(A\cdot X)_{[\struct,:~]}, v^* \big> > \mathbf{0}$}.
Here, the matrix-subscript $(A\cdot X)_{[\struct,:~]}$ denotes selecting the rows indexed by elements of the set $\struct$.
\end{lemma}
Since $G^*$ occurs in each graph of $\data_1$, we can set $v^* = \sum_{j\in G^*} v_j$, and the lemma holds naturally by construction.
We further note that there \textit{cannot} exist any other node, apart from $s$, that contains aggregated information from all the nodes of subgraph $G^*$, after one-step of GCN aggregation.

\begin{lemma}[Uniqueness of $v^*$] \label{lem:unique}
For all $w \in \real^d$ perpendicular to $v^*$, 
$w\cdot v^* = 0$, 
there exists a graph $G = (X,A) \in \data_1$ such that $\forall \struct \in [N]^M$, we have $\langle (A \cdot X)_{[\struct,:~]}, w \rangle \leq 0$ .
\end{lemma}
\noindent
The above lemma highlights that $G^*$ is the largest common-subgraph in $\data_1$, because all the smaller subgraphs of $G^*$ are also present in $G \in \data_1$ by the virtue of presence of $G^*$.
Any other subgraph, apart from $G^*$, will find its aggregated representation $w$ also in some $G \in \data_0$ graph.
$G^*$ being the largest common subgraph in $\data_1$, $v^*$ will not have any component along $w$.

For $G\in\data_1$, 
let the neighborhood of node $s$ comprise of nodes of $G^*$ in the set $\struct$ and other additional nodes $\{\ngbr{s}/\struct \}$. 
The aggregated representation of $s$ can be split into those coming from neighbors in $\struct$ and its complement set $\{\ngbr{s}/\struct \}$.
After attention based global pooling (readout), a graph $G \in \data_1$ can be represented as 
\begin{equation} \label{eqn:vhat}
h_G = \sum_{i\neq s, ~i=1}^N \alpha_i v_i + \alpha_s v_s = \sum_{i\neq s, ~i=1}^N \alpha_i v_i + {\alpha_s w_s} + {\alpha_s v^*}  
\end{equation}
Only the $\alpha_s v^*$ term is independent of the graph, as it is present in all graphs of $\data_1$, hence it is the only vector to have largest dot-product component along $v^*$.

\subsection{Loss Function and Gradient Flow of Parameters} \label{sec:setting}
The final GCN model's prediction $p\in [0,1]$ for the graph label is calculated as  \mbox{$p = \sigma ( \wgt\tp \aggr )$}, where 
$\sigma(x)$ is the sigmoid activation function. 
We characterize the evolution of the model parameters and their behavior by calculating their gradient flow \citep{DBLP:conf/nips/ChizatB18}.
For two class graph-classification, the model employs the binary cross entropy loss function \mbox{$\loss$: $- \big[ y \log(p) + (1-y)\log(1-p) \big]$}, where $y\in \{0,1\}$ is the true label and $p \in [0,1]$ is the model's predicted score.
The binary cross entropy loss function is given by
\begin{eqnarray*}
\loss &=& - \Big[ y \log(p) + (1-y)\log(1-p) \Big]  \\
&=&  \log\big(1+ \exp(\wgt\tp \aggr)\big) -y  \wgt\tp \aggr 
\end{eqnarray*}

\noindent
\textbf{Assumptions: }
In our theoretical analysis, we consider $G^*$ as a subgraph chain of radius one and derive results for a single-layer propagation of GCN, with attention based global pooling. 
During the course of optimization, we assume that the two partitions, $\data_1$ and $\data_0$, are of approximately similar size and we are able  to maintain mutual label parity to avoid the class-imbalance problem. 
For ease of notation, we assume that the grid-graphs have orthogonal node features $x_i$, which get mixed after a single-step of message propagation and give $v_i$. 
Due to the symmetry of grid-graphs, we upper bound dot-product of nodes $v_j\tp v_i$ by $\theta$. 
For bounded max-degree of the graphs, we represent $v_i\tp v_i = \theta_d$. By orthogonality among initial $x_i$, the relation $\theta_d \geq \theta$ follows.
By construction, partitions $\data_1, \data_0$ are exhaustive datasets that contain all the grid-graphs satisfying the requisite property. 

Over the course of optimization on the training data $\{\data_1 \cup \data_0\}$, we will next show that the weight parameter $\wgt$, as well as the attention parameter $a$ are most closely aligned along the direction of $v^*$. 
Proofs for the following results are deferred to the \textit{Appendix} \cite{pk2024gnnsoptimisation}.  

\begin{lemma}[Orientation of $\wgt$ parameter] \label{lem:vsdw}
There exists a vector $v^* \in \real^d$, such that the component of $\DD{\wgt}{t}$ during optimization is always aligned along $v^*$ and  
the weight parameter $\wgt$ moves in the direction of ${v^*}$,
$${v^*}\tp \DD{\wgt}{t} > 0 $$ 
\end{lemma}

\noindent
We further characterize the gradient flow of attention parameter $\attn$, and show that $\attn$ moves most along the direction $v^*$ after $O(D)$ update iterations, where $D$ is an upper bound on the maximum degree of graphs.
\begin{lemma}[Orientation of $\attn$ parameter]
There exists a vector $v^* \in \real^d$, such that the component of $\DD{\attn}{t}$ during optimization is always positive along $v^*$. 
After sufficient optimizer updates, the attention parameter $\attn$ moves most along the direction of ${v^*}$
$${v^*}\tp \DD{\attn}{t} > 0 $$     \label{lem:vsda}
\end{lemma}
%
%
%


\noindent
Based on the above lemmas, we can now state our main result  regarding the biased behavior of graph convolution networks with attention based global pooling. 
\begin{theorem}[Implicit Bias of Global Attention Pooling Networks]
\label{thm:main}
Consider the setting and assumptions of Section \ref{sec:setting},
and a 1-layer GCN coupled with attention based global pooling trained on $\data_1$ and $\data_0$. 
The parameters $\{\attn, \wgt\}$ at time $t$ iterations of gradient descent, satisfy 
$$ \big<\attn, v^* \big> > 0 ~~~~and~~~~  \big<\wgt, v^*\big> > 0 $$
for some $v^*$ per Lemma \ref{lem:exist} \mbox{$\big<(A\cdot X)_{[\struct,:~]}, v^* \big> > 0$}, $\forall (A,X)\in \data_1$.
\end{theorem}
\begin{hproof}
It follows from Lemma \ref{lem:vsdw} and \ref{lem:vsda} that parameters $\attn$ and $\wgt$ are poised to align with $v^*$ during optimization. 
Ensuring the existence of such $v^*$ using Lemma \ref{lem:exist} for training datasets $\data_1$ and $\data_0$ that are closed under completeness, it is easy to see that both the parameters will definitely have their component along $v^*$.
\end{hproof}

\emph{This is the main result of our work: it lays out an important property of GNN architectures that differentiate graphs based on the proximity of discriminative nodes found within the graph. }
Theorem \ref{thm:main} formalizes the biased nature of attention pooling and shows its preference to learn $v^*$, i.e. use  \textit{closely-connected-substructures} as discriminative features.
Consequently, \textit{average pooling}, which corresponds to uniform attention across all nodes, lacks the ability to selectively focus on specific nodes, thereby depending on aggregating nodes that may be distributed across the graph.

\section{Experiments}
\label{sec:exps}

To investigate the presence of inductive bias in diverse Graph Neural Network (GNN) architectures during gradient descent optimization, we pose the following pivotal question: 
Can two GNNs trained to achieve comparable accuracy on identical datasets ultimately learn characteristically different functions?
We show this distinction by examining their behaviors on  \textit{unseen data}, which the models have not come across while training. 
Focusing specifically on graph classification, we narrow our scope to ascertain if GNNs develop classification capabilities based on the presence/absence of a designated subgraph $G^*$. 
To address this, we create both synthetic and semi-synthetic datasets, by embedding $G^*$ within real-world datasets.

\subsection{Datasets and Model Setting}
\textbf{Synthetic:} 
The setting of grid-graphs, as seen in Section \ref{sec:OvC}, provides a structured framework for investigating implicit biases within a controlled domain. 
The options for substituting and embedding $G^*$ are inherently limited in this framework, due to the size of the grid that allows more control over its characteristics, including the balance between the positive and the negative graph instances. 
To elucidate the scenarios of extreme bias, it is imperative that the assumptions outlined in Section \ref{sec:setting} remain valid, a condition readily attainable through these grid-graph datasets. 

From Section \ref{sec:OvC} we know that the binary ground-truth labels for each partition are defined as $y_G = 1$, if $G\in\data_1$ and $y_G = 0$ if $G\in\data_0$.

\noindent
\textbf{Realworld (semi-synthetic):} 
Due to the absence of precise ground truth node-information pertaining to the subgraph $G^*$, we resort to augmenting some of the well-established graph datasets, such as ZINC \citep{DBLP:journals/corr/Gomez-Bombarelli16}, Tox21 
 from MoleculeNet \citep{DBLP:journals/corr/WuRFGGPLP17}, and PROTEINS from TU Dataset \citep{DBLP:journals/corr/abs-2007-08663}, with $G^*$ to create semi-synthetic datasets. 
This partial or complete planting of $G^*$ yields distinct partitions $\data_1$, $\data_0$, and $\dataP$. 
Our full-version  \cite{pk2024gnnsoptimisation} shows some of the samples of real-world-derived graphs underscoring their random structural compositions. 
Importantly, this diversity serves to mitigate any unintended influences stemming from the inherent symmetry in the grid graphs. 
Furthermore, these datasets will help identify the presence of implicit biases within GNNs when applied to real-world graph data.

\noindent
Planting $G^*$ in realworld data $\data$:~ 
We adopt a parallel approach for realworld datasets akin to our treatment of grid-graphs, and embed the $G^*$ subgraph either fully, partially, or sparsely. 
Details are deferred to the Appendix  \cite{pk2024gnnsoptimisation}.
Our code is available \cite{selfCode}.

\noindent
\textbf{Model baselines:}
In this study, we analyze GCN and GAT with standard global pooling (MAX,AVG,ATTN). 
Our focus lies in highlighting the phenomenon wherein GNNs that are trained on identical datasets, converge to different  functional representations. 
We substantiate this bias through concrete examples showcasing divergent behaviors, achieved with minimal adjustments made to the architecture or training paradigms.

\begin{table}
 \captionz{Real datasets used to create partitions $\data_1, \data_0, \dataP$ \label{tab:realworld-dataprop}}
    
    \renewcommand*{\arraystretch}{1.2}
    \centering
  \begin{tabular}{|l|c|c|c|c|}
  \hline
    Dataset &
      \begin{tabular}[c]{@{}c@{}} Average\\ [-0.2em] \#nodes\end{tabular} &
      \begin{tabular}[c]{@{}c@{}} Average\\ [-0.2em] \#edges\end{tabular} &
      \begin{tabular}[c]{@{}c@{}} Random \\ [-0.3em] colors\end{tabular} &
      \multicolumn{1}{c|}{\begin{tabular}[c]{@{}c@{}}\# Train \\ [-0.3em] graphs\end{tabular}}  \\    
  \hline
  \rule{0pt}{1.3em}
  \noindent
    ZINC \citep{DBLP:journals/corr/Gomez-Bombarelli16} & 23.2 & 49.8  & 24 & 10,000  \\
     Tox21 \citep{DBLP:journals/corr/WuRFGGPLP17}& 18.6 & 38.6  & 9  & 7,831  \\
     Proteins \citep{DBLP:journals/corr/abs-2007-08663}& 39.1 & 145.6 & 9  & 1,113    \\
    \hline
    \end{tabular}
\end{table}

\subsection{Performance on Real Datasets}\label{sec:readDatasets_perform}
To verify if the observations from Section \ref{sec:OvC} transcend to realworld scenarios, we 
perform experiments on four graph collection datasets, as described in Table \ref{tab:realworld-dataprop}. 
Appendix shows sample graphs used from these datasets while training \cite{pk2024gnnsoptimisation}. 
The skewed proportions of label assignment is also seen in real datasets. 
Table \ref{tab:all_results} shows GCN with ATTN pooling assigns majority of the labels as 1, thus classifies on the basis of connections.
This proportion is not as pronounced in AVG pooling because the classification is based on the presence of nodes as well. MAX pooling also has varied schemes of labeling among the three datasets. 
It is essential to acknowledge that these proportions are not as extreme as those observed in our toy dataset, owing to potential violations of assumptions such as data partition completeness. 
Nevertheless, the overarching trends persist, underscoring the robustness of our findings.

\begin{table}[th]
\captionz{ Real world dataset results, GCN 1 layer, trained to 100\% accuracy, averaged over multiple runs (standard deviation is order of magnitude smaller and so ignored)} 
   \label{tab:all_results}
   \renewcommand*{\arraystretch}{1.2}
   \centering
       \begin{tabular}{|l|ll|ll|ll|}
       \hline
       \multicolumn{1}{|c|}{\multirow{2}{*}{Dataset}} &
         \multicolumn{2}{c|}{MAX} &
         \multicolumn{2}{c|}{AVG} &
         \multicolumn{2}{c|}{ATTN} \\ \cline{2-7} 
       \multicolumn{1}{|c|}{} &
         \multicolumn{1}{c|}{$\hat{y}=1$} &
         \multicolumn{1}{c|}{$\hat{y}=0$} &
         \multicolumn{1}{c|}{$\hat{y}=1$} &
         \multicolumn{1}{c|}{$\hat{y}=0$} &
         \multicolumn{1}{c|}{$\hat{y}=1$} &
         \multicolumn{1}{c|}{$\hat{y}=0$} \\ \hline
       ZINC &
         \multicolumn{1}{l|}{0} &
         10000   &
         \multicolumn{1}{l|}{8561} &
         1439   &
         \multicolumn{1}{l|}{3} &
         9997    \\ \hline
       Protein &
         \multicolumn{1}{l|}{5531} &
         34 &
         \multicolumn{1}{l|}{1665} &
          3900 &
         \multicolumn{1}{l|}{6} &
          5559 \\ \hline
       Tox21 &
         \multicolumn{1}{l|}{0} &
          23493 &
         \multicolumn{1}{l|}{7576} &
         15917  &
         \multicolumn{1}{l|}{6} &
          23487 \\ \hline
       \end{tabular}
\end{table}

%

%

\subsection{Covering the Entire Attention-Spectrum}
\begin{table}[th]
\captionz{Classification behavior of 1-layer GCN and ATTN pooling, with varying temperature coefficient of softmax operator. All baselines are uniformly run for 100 epochs of optimization with learning rate of 0.001, test data is of size $10k$. \label{tbl:spectrum} }
 \renewcommand*{\arraystretch}{1.2}
    \centering 
\begin{tabular}{|c|c|c|c|}
\hline
$\beta$ Temperature & Train Accu. & $\hat{y}=0$ & $\hat{y}=1$ \\ \hline
1     & 0.9991      & 9806      & 194       \\ 
4     & 0.9982      & 9704      & 296       \\ 
10    & 0.993       & 9560      & 440       \\ 
300   & 1.0000      & \textbf{3}         & \textbf{9997}     \\ \hline
\end{tabular}
\end{table}
Attention pooling utilizes softmax to compute the weights for weighted summation. The softmax function can be modulated by a temperature coefficient, which governs the sharpness of the resulting distribution.
$softmax_\beta = \frac{\exp(\attn\tp v_i/\beta)}{\sum_j \exp(\attn\tp v_j/\beta)}$.
For $\beta$ set to 1, the original softmax function is recovered. However, for smaller values of $\beta$ approaching 0, the softmax tends towards behaving like a MAX pooling function. Conversely, as $\beta$ tends towards infinity, the softmax function becomes increasingly smoothed out and resembles an AVG pooling, akin to $1/N(\sum v_i)$. 
We investigate different values of $\beta$ to see the temperature point at which the inductive bias behavior switches. 
Table \ref{tbl:spectrum} shows that the flip in ATTN's behavior happens at $\beta=300$. Until then, ATTN pooling behaves much the same. Intuitively, this reinforces the separation between zero-attention (AVG) and non-zero attention (ATTN).

\subsection{Delineating the Effect of Linear Layer } \label{sec:linear_layer}

\begin{table}[]
\captionz{Effect of pruning out linear classifier and still observing consistent label assignment trends (Sec.\ref{sec:linear_layer}) \label{tbl:without_linear} }
 \renewcommand*{\arraystretch}{1.2}
    \centering 
\begin{tabular}{|c|rr|rr|}
\hline
\multirow{2}{*}{\begin{tabular}[c]{@{}c@{}}Pooling\\ Method\end{tabular}} &
  \multicolumn{2}{c|}{With Linear Layer} &
  \multicolumn{2}{c|}{Without Linear Layer} \\ \cline{2-5} 
 &
  \multicolumn{1}{c|}{$\hat{y}=0$} &
  \multicolumn{1}{c|}{$\hat{y}=1$} &
  \multicolumn{1}{c|}{$\hat{y}=0$} &
  \multicolumn{1}{c|}{$\hat{y}=1$} \\ \hline
MAX  & \multicolumn{1}{r|}{9988} & 12    & \multicolumn{1}{r|}{9585} & 415  \\ \hline
AVG  & \multicolumn{1}{r|}{0}    & 10000 & \multicolumn{1}{r|}{2248} & 7752 \\ \hline
ATTN & \multicolumn{1}{r|}{9940} & 60    & \multicolumn{1}{r|}{9465} & 535  \\ \hline
\end{tabular}
\end{table}

Equation.\ref{eqn:GCN_ATTNupdate} provides the formulation for end-to-end graph classification, employing attention-based global pooling within a single-layer GCN. 
It can appear that the biases discussed in preceding sections stem from the presence of the linear layer. 
While the GCN convolution layer facilitates message passing, the subsequent pooling layer condenses all nodes into a unified graph representation, followed by the linear layer assigning a logic for classification. 
To disentangle the influence of the linear layer from the GCN, we opt to prune out the linear classifier and allow the GCN to perform the necessary dimension reduction, resulting in $\Theta\in \real^{N\times 1}$. Subsequently, we look into the alterations in graph classification behavior resulting from the removal of the linear classifier.

Table \ref{tbl:without_linear} illustrates that the proportion of nodes classified as $\hat{Y}=0$ remains consistent irrespective of the presence or absence of the linear layer. 
Upon closer examination, it becomes evident that the absence of the linear layer leads to elevated training loss, potentially indicating that the dimension of parameters $\Theta \in \mathbb{R}^{N\times 1}$ might be insufficient to effectively capture meaningful graph representations.
Nonetheless, the initial bias persists in the distribution of nodes across each label. 
This observation underscores the intrinsic bias inherent in the layers of GNNs, even after pruning the linear classifier.

\subsection{Effect of Hierarchical Pooling}\label{sec:heirarchical_pooling}

\begin{table}[ht]
\captionz{Classification behavior with hierarchical ASAP pooling.  \label{tbl:heirarchical} }
 \renewcommand*{\arraystretch}{1.2}
    \centering 
\begin{tabular}{|l|l|l|l|l|}
\hline
\begin{tabular}[c]{@{}l@{}}Hierarchical\\ Pooling\end{tabular} &
  \begin{tabular}[c]{@{}l@{}}Global\\ Pool\end{tabular} &
  \begin{tabular}[c]{@{}l@{}}Training\\ Accu.\end{tabular} &
  $\hat{y}=0$ &
  $\hat{y}=1$ \\ \hline
\multirow{3}{*}{\begin{tabular}[c]{@{}l@{}}GCN + ASAP \\ Pooling\end{tabular}} & ATTN & 0.9875 & 8736 & 1264 \\ \cline{2-5} 
                                                                                 & AVG  & 0.9825 & 8519 & 1481 \\ \cline{2-5} 
                                                                                 & MAX  & 0.9725 & 8491 & 1509 \\ \hline
\end{tabular}
\end{table}

Global pooling methods operate simultaneously on all the nodes within a graph, thereby failing to encode its structural nuances. 
While this aggregation process condenses all nodes into a singular representation within one single time-step, in contrast, 
Hierarchical pooling methods progressively coarsen the graph across multiple iterations, thereby preserving its hierarchical structure. 
We scrutinize the behaviors of select hierarchical pooling operators, such as ASAP \citep{DBLP:conf/aaai/RanjanST20}. 
Intuitively, these operators are expected to identify and prioritize significant subgraphs for classification purposes. 
Table \ref{tbl:heirarchical} 
 shows a striking resemblance in the behaviors of hierarchical pooling operators to that of attention-based global pooling. 
A plausible explanation for this similarity lies in their shared objective of discerning connected patterns to effectively aggregate graph information, thereby exhibiting analogous behaviors.

\subsection{Discussion}

Our work has practical implications in tasks where domain knowledge guides graph classification. 
It can help distinguish whether classification depends on coherent subgraph presence or merely node presence, regardless of neighbors. 
For instance, in chemistry, chemical reaction likelihood may hinge on specific node configurations, while compound toxicity could relate to certain heavy metal atoms. Thus, attention-based global pooling suits chemical reactions, while average pooling may be apt for detecting poisonous heavy metals.

We opted to analyze Attention pooling since it can be parameterized by adjusting the temperature coefficient of the sigmoid operator, Attention pooling can be analyzed across a spectrum ranging from MAX to AVG, with ATTN lying in between. 
The theoretical results hold for any graph $G^*$ with diameter 2 (i.e., star-graphs). We have depicted using 3-node graphs for simpler illustrations, but the conclusions remain unchanged for larger star-graphs. 
Empirically we find that the same trends and conclusions hold for $G^*$ with a larger diameter.

\section{Related Works} 
\label{sec:relatedworks}

Inductive biases are common in machine learning, and they shape the learned functions. Architecture-induced biases stem from design choices, enhancing performance on specific data types \citep{DBLP:journals/corr/abs-2204-07697}. 
For instance, convolutional neural networks with max-pooling exhibit translation invariance over images, while graph neural networks are permutation invariant in node ordering over graphs \citep{DBLP:journals/corr/abs-1801-01450,DBLP:conf/nips/KerivenP19}. 
Optimizer-induced biases can arise in learned models due to factors such as the loss function's nature, parameter initialization, and regularization \citep{DBLP:journals/corr/abs-2101-00072}.
In this work we characterize the nature of pooling mechanisms in learning different kinds of subgraph patterns.

Graph global pooling operators implement functions to incorporate node's information into concise and reduced graph representation.
After sufficient layers of message-passing, there are simple single-step operations such as \texttt{max, sum, avg} that take all nodes at once to derive graph's representation \citep{DBLP:conf/nips/DuvenaudMABHAA15,DBLP:conf/nips/BianchiL23}. 
In similar spirit, global attention pooling can lay emphasis on certain \textit{important} nodes and arrive at graph's representation by asymmetric weighted addition \citep{DBLP:conf/icml/LeeLK19}.
There are other stepwise pooling methods, such as ASAP, Top-k \citep{DBLP:conf/aaai/RanjanST20,DBLP:conf/icml/GaoJ19}, that coarsen the graph structure in several steps.
In our work, we focus on characterizing the implicit bias of single step global pooling mechanisms.
To the best of our knowledge, there aren't any previous works that evaluate implicit inductive bias of GNNs to preferentially learn certain functions for graph classification tasks.
There is a large gap to be filled with practical considerations of optimization and learnability. 
MPNNs can count a context of subgraph-counting, for 3 or fewer nodes in the subgraph (except star-shaped subgraphs) \citep{DBLP:conf/nips/Chen0VB20}. 
While there is a great body of work on representation power of GNNs, there is a large gap with its practical optimization and learnability property. 

%

\section{Conclusion}
We highlighted some of the gaps between the representation theory and optimization aspect of GNNs. 
We showed the existence of implicit inductive bias in GNNs, specifically we proved the bias of graph convolution networks with attention based global pooling. 
We showed theoretically and empirically the preference of attention based architectures, to look for a \textit{closely-connected} patterns for graph classification. 
We discussed the implications of our work in incorporating domain knowledge.

The implicit bias herein represents just one facet, namely, the preferential treatment of nodes occurring together or dispersed across the graph. 
Other biases may exist, such as propensity to assign different labels to homophilic/heterophilic graphs versus treating them uniformly. Investigating these nuanced relationships between GNNs, optimizers, and real-world applications are avenues for future exploration.

\begin{ack}
We gratefully acknowledge the support of Robert Bosch Centre for Data Science \& Artificial Intelligence (RBCDSAI) at Indian Institute of Technology Madras, as well as our anonymous reviewers for providing valuable feedback.
\end{ack}

\bibliography{refs.bib}


\supplementaryOn{

\clearpage
\onecolumn

\begin{table}[!h]
\renewcommand*{\arraystretch}{1.4}
\captionz{Table of notations used in the paper}
\begin{tabular}{l|l}
\hline
$\data$ &  Set of all graphs $(X,A)$ of fixed size $|X|=N$\\
$G \in \data $ & $(X,A)\in (\real^{N \times d}, \real^{N \times N})$ Graph with feature and adjacency matrices \\
$G^*$  & $(X^*,A^*)\in (\real^{M \times d}, \real^{M \times M})$ Graph with $|X^*|=M$ nodes, where $M<N$ \\
$[N]$ & Set of natural numbers upto $N$, i.e, $\{1,2,3 \ldots N\}$ \\
$\struct $ &  $\in [N]^M = (s_1, s_2, \ldots s_M)$ ordered tuple of indices in $[N]$ and cardinality $M$ \\ 
$X_{[\struct,:]}$ & $\in \real^{M\times d}$ contains the rows of $X$ corresponding to indices in ordered tuple $\struct$ with  $|\struct|=M$ \\
$A_{[\struct,\struct]}$ &  $\in \real^{M\times M}$ contains rows and columns corresponding to indices in ordered tuple $\struct$ with  $|\struct|=M$  \\
$\data_1$  &  $\{G \in \data ~\vert~ \exists \struct = (s_1, s_2 \ldots s_M) \in [N]^M\}$, such that $X_{[\struct,:]} = X^*$ and $A_{[\struct,\struct]} \geq A^*$ \\
& for some ordered tuple $\struct$. Here matrix comparison is element-wise\\
$\data_0$  &   $\{G \in \data ~\vert~ X_\struct \neq X^*, ~\forall \struct \in [N]^M\}$\\
$\data_\perp$  &   $\{G \in \data ~\vert~  ~\exists \struct\in [N]^M ~\text{s.t}~ X_\struct = X^* ~\text{and}~ G \notin D_1\}$\\
\hline
For a given graph G & \\
$y_G \in \{0,1\}$  & True label of graph~ $y_G = 1, ~\text{if}~ G\in\data_1$, ~~ $y_G = 0, ~\text{if}~ G\in\data_0$\\
$N$  & Number of nodes in graph \\
$d = \texttt{maxdegree}(G)$ & Maximum degree among all nodes in the graph \\
$x_i = X_{i,:}$ & $\in \real^d$ feature vector of node $i$\\
$V = AX$ & $ \in \real^{n \times d}$ feature matrix after 1-hop GCN aggregation \\ 
$v_i = \sum\limits_{j \in \ngbr{i}} x_j$ & $= V_{i,:}$ feature vector of node $i$ after 1-hop GCN aggregation\\ 
$\alpha_i = \mathrm{softmax}_i(\attn\tp v_i)$ &  $\in [0,1]$ normalized attention score for each node $i$  in the Graph \\
$\aggr = \sum\limits_{i=1}^N \alpha_i v_i$  & $\in \real^d$  weighted sum of all nodes in the graph\\
$p = \sigma ( \wgt\tp \aggr ) $ &  $\in [0,1]$ Model's prediction \\
$\loss=$  BCE Loss & $- \Big[ y \log(p) + (1-y)\log(1-p) \Big]$ \\
$\psi_s$ & ${v^*}\tp\wgt$ \\
\hline
$\attn \in \real^d$ & Attention parameter \\
$\wgt \in \real^d$ & Linear classifier parameter \\
\end{tabular}
\end{table}

\section{Appendix}

To quantify the evolution of model parameters with respect to time, we 
analyze the change in loss function with respect to pre-activation logits and $\wgt$, 
\begin{eqnarray}
\dd{\loss}{~\wgt\tp \aggr} &=& \sigma( \wgt\tp \aggr)-y \nonumber  \\
&=& p-y \nonumber \\
- \dd{\loss}{\wgt} &=& -\dd{\loss}{~\wgt\tp \aggr} \cdot \dd{~\wgt\tp \aggr}{\wgt} \nonumber \\
&=&  (y-p) \aggr  \label{eqn:dwdt}
\end{eqnarray}

\noindent
Similarly, for the rate of change of parameter $\attn$
\begin{align*}
- \dd{\loss}{\attn} &=&& -\dd{\loss}{~\wgt\tp \aggr} \cdot \dd{~\wgt\tp \aggr}{\attn}   \\
&=&&  (y-p) \cdot \dd{\big[\wgt\tp \sum_{i=1}^N \alpha_i v_i  \big]}{\attn}   ~ \mathrm{here}~ \alpha_i = \mathrm{softmax}_i (\attn\tp v_i )  \\
&=&& (y-p) \cdot \wgt\tp \dd{~}{\attn}\left[ \frac{\exp(\attn\tp v_i)}{Z} v_i\right] ~\mathrm{here} ~ Z = \sum\limits_{i=1}^N e^{(\attn\tp v_i)} \\
&=&& (y-p) \cdot \wgt\tp \left[ \frac{\exp(\attn\tp v_i)~ v_i v_i\tp}{Z} - \frac{\exp(\attn\tp v_i)}{Z} \frac{\exp(\attn\tp v_i\tp)}{Z}  \right]   
\\
&=&& (y-p) \cdot \sum_{i=1}^N (v_i - \aggr)~v_i\tp   \cdot \wgt   \numberthis \label{eqn:dadt}
\end{align*}

\noindent
\textbf{At initialization}   
the parameters are assigned values $\attn = 0, \wgt=0$ (at time $t=0$). 
For the entire dataset $\data$, we have the following conditions at the beginning of optimization
\begin{eqnarray*}
\aggr &=&  \frac{1}{N} \sum_{i=1}^N v_i ~~~ \text{because $\alpha_i$ = softmax$(\mathbf{0}\tp v_i )$} \\
\DD{\wgt}{t} &=& (y-p)\cdot \frac{1}{N} \sum_{i=1}^N v_i  \\
\DD{\attn}{t} &=& (y-p) \cdot \sum_{i=1}^N (v_i - \aggr)~v_i\tp   \cdot \mathbf{0}\\
&=&  0 \\
\end{eqnarray*}

\subsubsection*{Proof of Lemma \ref{lem:vsdw}}

\textbf{Lemma \ref{lem:vsdw}}~[Orientation of $\wgt$ parameter]
\textit{There exists a vector $v^* \in \real^d$, such that the component of $\DD{\wgt}{t}$ during optimization is always aligned along $v^*$ and  
the weight parameter $\wgt$ moves in the direction of ${v^*}$,
$${v^*}\tp \DD{\wgt}{t} > 0 $$  }
\begin{proof}
Following from Equation.\ref{eqn:dwdt}, we substitute $\hat{v} =  \sum_{i=1}^N \alpha_i v_i$ and separate out the term containing $v^*$.
Note that our assumption of data completeness implies that while expanding $\hat{v}$, the additional terms of $w_s$, as seen in Equation.\ref{eqn:vhat}, is found in both $\data_1$ as well as $\data_0$. 
Thus, due to the presence of $w_s$ in both label $y=1$ and $y=0$, the optimizer would cancel the update. Only $v^*$ emerges to be present in all $\data_1$ graphs, and simultaneously absent in $\data_0$. 
\begin{eqnarray}
v\tp \DD{\wgt}{t} &=& (1-p)\Big[ \alpha_s \vjv{v^*} + \sum_{i \neq s} \alpha_i \vjv{v_i} \Big] - p \sum_{i \neq s} \beta_{i} \vjv{v_i}  \nonumber \\
&= & (1-p)\Big[ \alpha_s \theta + \sum_{i \neq s} \alpha_i \theta \Big] - p  \sum_{i\neq s} \beta_i \theta 
\nonumber \\
&=& (1-p)\Big[ \alpha_s \theta + \sum_{i \neq s} \alpha_i\theta \Big] - p \theta   \label{eqn:vjdwdt}
\end{eqnarray}
The term $\beta_i v_i\tp v_i$ in Equation.\ref{eqn:vjdwdt} is absorbed in $\beta_i \theta$ due to $\alpha_s > \beta_i$.
Full proof, without this simplifying assumption, along with expanded $\beta_i$ terms are provided in supplementary material.
 Similarly we expand the gradient's component along $v^*$ to get, 
\begin{eqnarray}
{v^*}\tp \DD{\wgt}{t} &=& (1-p)\Big[ \alpha_s \vstarv{{v^*}} + \sum_{i \neq s} \alpha_i \vstarv{v_i} \Big] - p  \sum_{i \neq s} \beta_i \vstarv{v_i}  \nonumber \\
& = &  (1-p)\Big[ \alpha_s \theta_d + \sum_{i \neq s} \alpha_i \theta \Big] - p  \sum_{i \neq s} \beta_i \theta 
\nonumber \\
&=& (1-p)\Big[ \alpha_s \theta_d + \sum_{i \neq s} \alpha_i \theta \Big] - p \theta \label{eqn:vsdwdt} 
\end{eqnarray}
Further, for a well trained GNN model with low training loss, the prediction probability satisfies $p<0.5$ in Eqn.\ref{eqn:vjdwdt}, making $1-2p > 0$. 
Given that $\theta_d > \theta$, upon comparing  Eqn.\ref{eqn:vjdwdt} and Eqn.\ref{eqn:vsdwdt},  we obtain 
${v^*}\tp \DD{\wgt}{t} \geq v\tp \DD{\wgt}{t} > 0$, thus proving the desired comparison.
\end{proof}

\subsubsection*{Proof of Lemma.\ref{lem:vsda}}

\textbf{Lemma.\ref{lem:vsda}}~[Orientation of $\attn$ parameter]
\textit{There exists a vector $v^* \in \real^d$, such that the component of $\DD{\attn}{t}$ during optimization is always positive along $v^*$. 
After sufficient optimizer updates, the attention parameter $\attn$ moves most along the direction of ${v^*}$, i.e., for any other $v$ of magnitude, $\Vert v \Vert = \Vert v^* \Vert$, we show that 
$${v^*}\tp \DD{\attn}{t} > 0 $$ }

\begin{proof}
Let the assumptions of data-completeness, GNN model, label-imbalance etc., as mentioned in Section \ref{sec:setting} still hold.
Let $\psi_s = {v^*}\tp\wgt$ and $\psi = v_i\tp\wgt$.
\begin{eqnarray}
{v^*}\tp \DD{\attn}{t} &=& (1-p)\Big[ \alpha_s (\vstarv{v^*} - {v^*}\tp \aggr) {v^*}\tp \wgt + \sum_{i\neq s} \alpha_i (\vstarv{v_i} - {v^*}\tp \aggr)v_i\tp \wgt  \Big] - p \Big[ \sum_i \beta_i (\vstarv{v_i} - {v^*}\tp \aggr)v_i\tp\wgt   \Big] \nonumber \\
&=&  (1-p) \Big[ \alpha_s\psi_s (\vstarv{v^*} - {v^*}\tp \aggr)  + \sum_{i\neq s} \alpha_i\psi (\vstarv{v_i} - {v^*}\tp \aggr)  \Big]  -p \Big[ \sum_i \beta_i\psi(\vstarv{v_i} - {v^*}\tp\aggr) \Big]    \nonumber    \\
&=&  (1-p) \Big[ \alpha_s \psi_s (\theta_d - \alpha_s\theta_d - (1-\alpha_s)\theta)  + \sum_{i\neq s} \alpha_i\psi(\theta - \alpha_s\theta_d - (1-\alpha_s)\theta)\Big] -p  \sum_i \beta_i \psi (\theta - \theta)   \nonumber  \\
&=&  (1-p) \Big[ \alpha_s \psi_s (1-\alpha_s)(\theta_d - \theta) -(1-\alpha_s)\psi \alpha_s (\theta_d-\theta)  \Big]   \nonumber  \\
&=&  (1-p)\alpha_s  (1-\alpha_s) (\psi_s-\psi) (\theta_d-\theta)  \nonumber   \\
&=&  (1-p)\alpha_s (1-\alpha_s) \psi (q-1) (\theta_d-\theta)   \label{eqn:vsdadt}
\end{eqnarray}

We can see that $\psi_s \geq \psi$, 
now let $\psi_s = q\cdot \psi$, for some scalar $q\geq 1$. 
\begin{eqnarray}
v\tp \DD{\attn}{t} &=&  (1-p)\Big[ \alpha_s (\vjv{v^*} - v\tp \aggr) {v^*}\tp \wgt + \sum_{i\neq s} \alpha_i (\vjv{v_i} - v\tp \aggr)v_i\tp \wgt  \Big] - p \Big[ \sum_i \beta_i (\vjv{v_i} - v\tp \aggr)v_i\tp\wgt   \Big] \nonumber  \\
&=&  (1-p)\Big[ \alpha_s \psi_s (\vjv{v^*} - v\tp \aggr) + \sum_{i\neq s} \alpha_i \psi (\vjv{v_i} - v\tp \aggr) \Big]  -p \Big[ \sum_i \beta_i \psi  (\vjv{v_i} - v\tp \aggr) \Big]  \nonumber \\
&=&  (1-p)\Big[ \alpha_s \psi_s (\theta  -\alpha_j\theta_d -(1-\alpha_j)\theta) + \sum_{i\neq s} \alpha_i \psi (\theta  -\alpha_j\theta_d -(1-\alpha_j)\theta)     \Big]  -p \Big[ \sum_i \beta_i \psi (\theta - \beta_j \theta_d - (1-\beta_j)\theta) \Big]  \nonumber \\
&=&  (1-p)\left[ \alpha_s\psi_s \alpha_j (\theta -\theta_d)  + (1-\alpha_s)\psi \alpha_j (\theta -\theta_d)  \right]  -p\left[\psi \beta_j (\theta -\theta_d) \right]    \nonumber\\
&=&  (1-p)\big( \psi \alpha_j (\theta -\theta_d)  (\alpha_s q  + 1-\alpha_s) \big)  -p\psi \beta_j (\theta -\theta_d)     \label{eqn:vjdadt}
\end{eqnarray} 

From Eqn.\ref{eqn:vsdadt} and Eqn.\ref{eqn:vjdadt}, we note that 
\begin{equation}
{v^*}\tp \DD{\attn}{t} > v\tp \DD{\attn}{t} ~~~\text{if}~~ q > \frac{\beta_j -\alpha_j}{\alpha_s(1-\alpha_s)+ \alpha_j \alpha_s} + 1~~; j\in [N] \label{eqn:condition}
\end{equation} 
 where $q = (\psi_s / \psi)$. We know from Lemma \ref{lem:vsdw} that ${v^*}\tp \DD{\wgt}{t} \geq v_j\tp \DD{\wgt}{t} > 0, ~~\forall j \in [N]/s$. 
Since the component of $\wgt$ along ${v^*}$  monotonically increases, because $\wgt$ moves in its direction of more than any other $v_j$, we can see that the condition in Eqn.\ref{eqn:condition} will be satisfied after sufficient iterations.
\end{proof}

\subsection*{Additional Relations}
\begin{eqnarray*}
{v^*}\tp v^* &=&  \theta_d  \\
{v^*}\tp v_i &=&  \theta  \\
\sum_{i=1}^N \alpha_i &=& 1 
\end{eqnarray*}

For graph with label $y=1$
\begin{eqnarray*}
{v^*}\tp \aggr &=& {v^*}\tp (\alpha_s {v^*} + \sum_{i\neq s} \alpha_i v_i) \\
    &=& \alpha_s \theta_d + (1-\alpha_s)\theta \\
\\
v_j\tp \aggr &=& v_j\tp \Big(\alpha_s {v^*} + \alpha_j v_j + \sum_{i\neq s,j} \alpha_i v_i\Big) \\
      &=& \alpha_s\theta + \alpha_j\theta_d + \sum_{i\neq s,j} \alpha_i \theta \\
      &=& \alpha_j \theta_d + (1-\alpha_j)\theta \\
\end{eqnarray*}

For graph with label $y=0$
\begin{eqnarray*}
{v^*}\tp \aggr &=& \sum_{i} \beta_i {v^*}\tp v_i \\
    &=& \theta \\
\\
v_j\tp \aggr &=& v_j\tp \Big(\beta_j v_j + \sum_{i\neq j} \beta_i v_i\Big) \\
      &=& \beta_j \theta_d + (1-\beta_j)\theta \\     
\end{eqnarray*}

\begin{lemma}[Learning Updates Preserve the Assumptions] \label{lem:preserve}
For all $v_i \in \real^d$, such that  $v = \sum_{i=1}^N \eta_i x_i $, where $\eta_i \in \mathbb{Z}_+$ along with condition $\sum_{i=1}^N \eta_i \leq d$ max-degree, and $x_i$ are initial node-representations,  
the validity of both the conditions \textit{(i)} $\wgt\tp {v^*} >\wgt\tp v_i$  and \textit{(ii)} $\attn\tp {v^*} > \attn\tp v_i,  ~~\forall i \in [N] $ in gradient flow  of $\attn$ and $\wgt$ is upheld throughout their learning trajectory.  
\end{lemma}
\begin{proof}
Extending the observations made in Lemma.\ref{lem:vsdw} and  \ref{lem:vsda} above, at time $t$, to the next time-step $t+1$ of optimization, shows that the validity of assumption holds naturally.
\begin{eqnarray*}
\left[ \wgt + \DD{\wgt}{t} \right]\tp {v^*} & > & \left[ \wgt + \DD{\wgt}{t} \right]\tp v_i   \\
\wgt\tp {v^*} + \DD{\wgt}{t}\tp {v^*}  &>&  \wgt\tp v_i + \DD{\wgt}{t}\tp v_i  ~~~~~\text{from (Lemma.\ref{lem:vsdw})}\\
\left[ \attn + \DD{\attn}{t} \right]\tp {v^*} & > & \left[ \attn + \DD{\attn}{t} \right]\tp v_i   \\
 \attn\tp {v^*} + \DD{\attn}{t}\tp {v^*} & > &  \attn\tp v_i + \DD{\attn}{t}\tp v_i   ~~~~~\text{from (Lemma.\ref{lem:vsda})}
\end{eqnarray*}
\end{proof}

\begin{figure*}[t]
    \centering
       \includegraphics[width=0.8\textwidth]{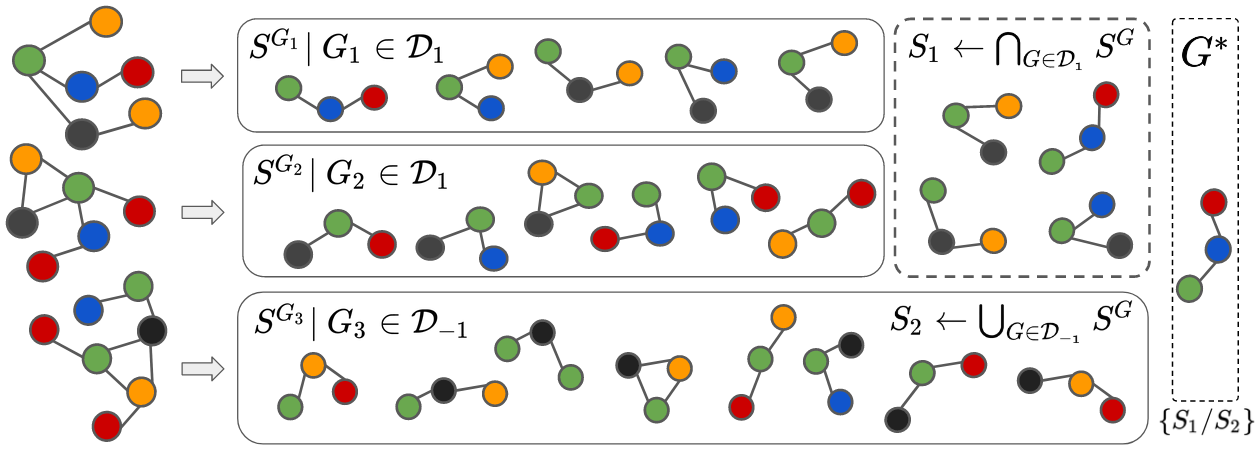}
        \captionz{Pictorial depiction of the exhaustive-subgraph-search Algorithm.\ref{alg:cap}. First, we take intersection of all subgraphs of size $<k$, belonging to $\data_1$, and subsequently remove any subgraph that is found in $D_{-1}$}. \label{fig:alg}
\end{figure*}

\subsection{Exhaustive Subgraph Search Algorithm}
Given training datasets as described in Section \ref{sec:subgraph-find}, the following algorithm can find a subgraph $G^*$ that will be the discriminating feature between $\data_1$ and $\data_{-1}$. Figure \ref{fig:alg} shows the step-wise computations for the same. 

\begin{algorithm}
\caption{Graph classification based on exhaustive search for presence of a subgraph}\label{alg:cap}
\begin{algorithmic}
\Require Dataset $\data = \{\data_{1} \cup \data_{-1} \}$, Upper bound $k$ on number of nodes in $G^*$ 

\State $S^G \gets$ Set of all subgraphs of $G$ of size less than $k$
\State $S_1 \gets$ Intersection of subgraphs in $\{S^G\}$ where $G\in \data_1$
\State $S_0 \gets$ Union of subgraphs in $\{S^G\}$ where $G\in \data_0$
\State $G^* \gets$  Set difference $\{S_1\} \backslash \{S_0\} $

\end{algorithmic}
\end{algorithm}

\subsection{Planting $G^*$ in realworld data $\data$:} \label{appdx:realdata}

For any given graph $G$ sourced from a real dataset $\data$, we introduce chain subgraphs $G^*$ spanning sizes ranging from 3 to 5 nodes.
For $\data_1$, we randomly select anchor nodes smaller in size than $G^*$, and establish connections between these nodes and $G$. Conversely, in $\data_0$, we take two subsets $S_1$ and $S_2$ of nodes within $G^*$, such that their union remains a proper subset of $G^*$, ensuring at least one node exclusively belongs to $G^*$ and not to any of the planted subsets. 
We further choose random anchor nodes to facilitate  edge connections.
We initialize $\data_1$ first, followed by constructing $\data_0$, thus we preserve dataset size parity.
Maintaining balance in class distribution is crucial during the creation of training data, as it ensures equal representation within $\data_1, \data_0$, and mitigates label-imbalance issue. 
Subsequently, we form $\dataP$ by selecting non-neighboring nodes as anchors for each node within $G^*$. 
An important distinction here is that these generated graphs remain unseen during training, thereby eliminating constraints on the size of $\dataP$. Table \ref{tab:realworld-dataprop} presents key statistics concerning the generated data partitions.
The \textit{Random-colors} in the table depict the number of distinct node-features overall in the dataset, which includes the distinct colors added by the implanted $G^*$. 
For example ZINC dataset comprises of molecule-graphs that are made up of 21 kinds of atoms (carbon, oxygen, other heavy atoms), and we further add $G^*$ with 3 other kinds of nodes.

Our test data volume exceeds that of the training data by an order of magnitude, to maintain high statistical accuracy. 
The feature space of the dataset is encapsulated by the metric "Random-colors", indicating the total count of distinct node features, inclusive of those introduced by the embedded $G^*$. For instance, in the ZINC dataset comprising molecular graphs characterized by 21 atom types, we augment $G^*$ with 3 additional node types. Conforming to our prior methodology, we construct a subgraph $\mathcal{G}^*$ consisting of 3 connected nodes with distinct features. Embedding $\mathcal{G}^*$ involves randomly establishing edges between any two nodes within $\mathcal{G}$ and $\mathcal{G}^*$, subsequently labeling it as $y=1$. Similarly, we generate counterpart labels ($y=0$) by incorporating only a subset of nodes $\hat{\mathcal{G}} \subset \mathcal{G}^*$ into ZINC graphs.

\begin{figure}[t!]
    \centering 
    \begin{subfigure}[t]{0.26\textwidth}
        \centering
        \includegraphics[width=0.9\textwidth]{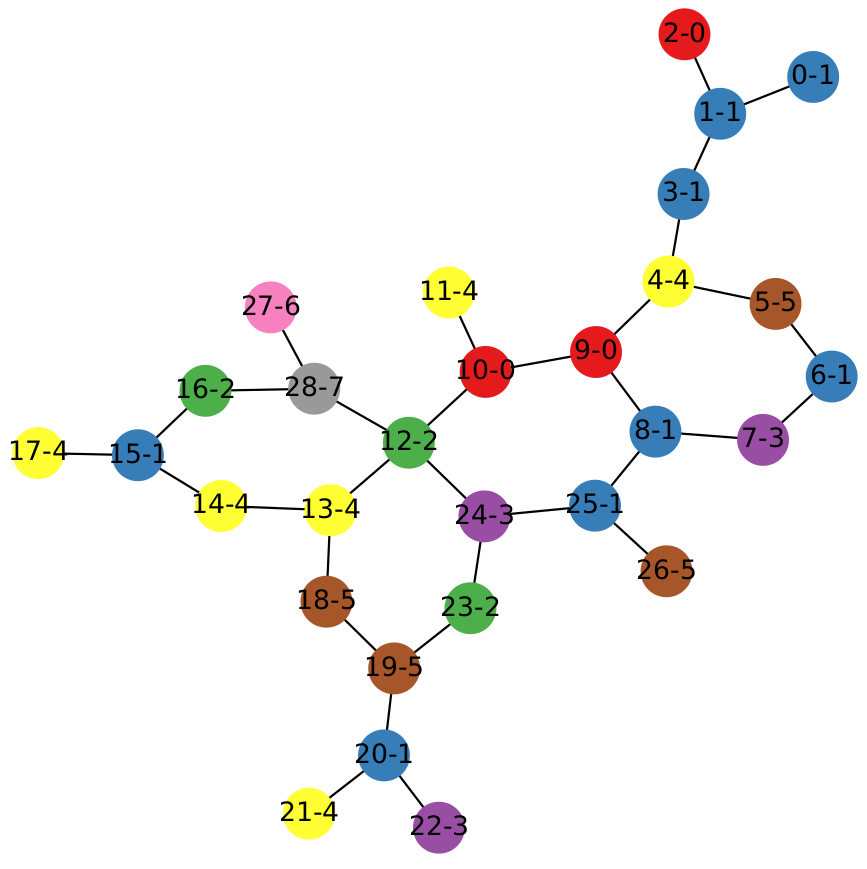}
        \captionz{Tox 21} \vspace{3mm} \label{fig:realdata.g2}
    \end{subfigure} 
     \begin{subfigure}[t]{0.26\textwidth}
        \centering
        \includegraphics[width=0.85\textwidth]{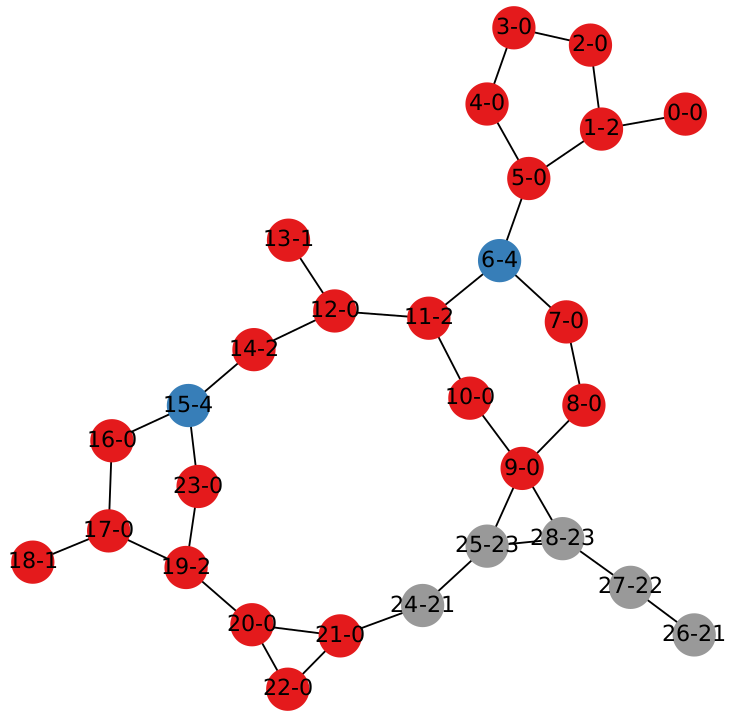}
        \captionz{ZINC} \label{fig:realdata.g3}
    \end{subfigure}
    \begin{subfigure}[t]{0.25\textwidth}
        \centering
        \includegraphics[width=0.85\textwidth]{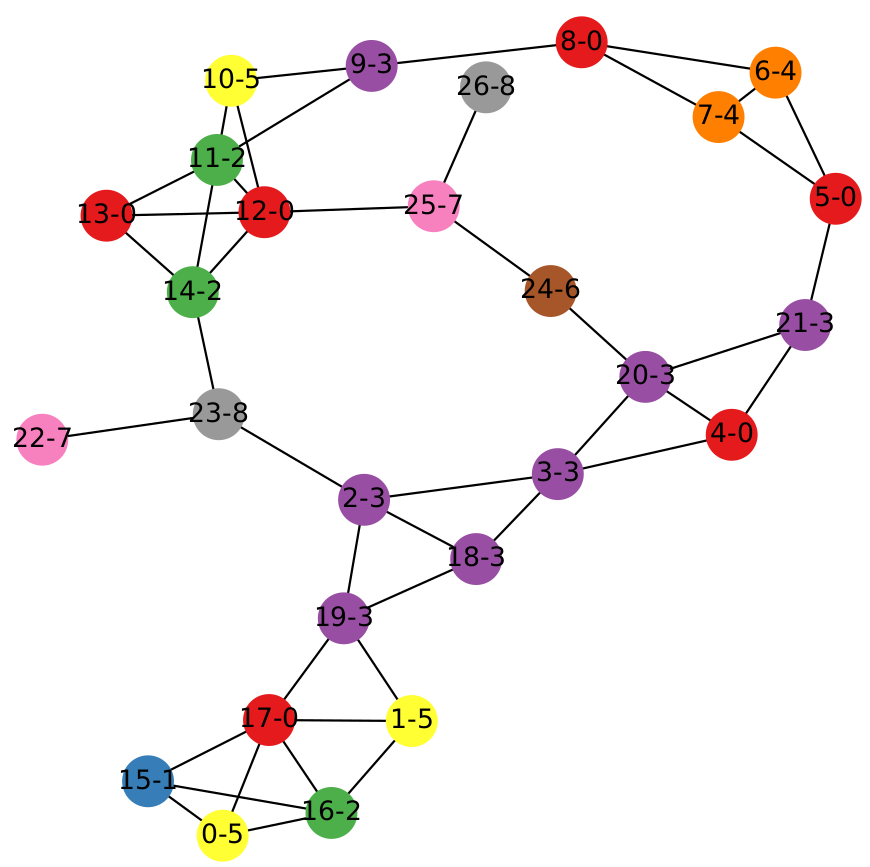}
        \captionz{Protein} \label{fig:realdata.g5}
    \end{subfigure}
    \captionz{Sample graphs from real-world datasets that are planted with $G^*$, both with occurrence and connected-ness properties. Different colors depict the node features. 
           \label{fig:realdata}
          }
          \vspace{3mm}
\end{figure}

}

\end{document}